\documentclass[conference]{IEEEtran}
\IEEEoverridecommandlockouts
\usepackage[T1]{fontenc}
\usepackage[utf8]{inputenc}
\usepackage{times}
\usepackage{microtype}
\usepackage{inconsolata}
\usepackage{graphicx}
\usepackage{booktabs}
\usepackage{amsmath,amssymb,amsthm}
\usepackage{xcolor}
\usepackage[numbers,sort&compress]{natbib}
\usepackage{url}
\usepackage{xurl}
\usepackage[hidelinks]{hyperref}
\newtheorem{proposition}{Proposition}
\makeatletter\renewcommand{\paragraph}[1]{\par\noindent\textbf{#1}\hspace{0.5em}}\makeatother

\renewenvironment{abstract}{\par\noindent\footnotesize\bfseries\textit{Abstract.}\ }{\par\vspace{0.6\baselineskip}}
\newcommand{\pL}{\texttt{\textbackslash p\{L\}}}
\newcommand{\pM}{\texttt{\textbackslash p\{M\}}}
\newcommand{\RO}{\textsc{R0}}
\newcommand{\RI}{\textsc{R1}}
\newcommand{\RII}{\textsc{R2}}
\newcommand{\artifactnote}{Code, results and the analysis plan: \url{https://github.com/arjanchaudharyy/nepali-pretokenizer-repair}. Tokenizers (\RII{}, then \RI{}): \url{https://huggingface.co/Aarjan/Llama-3.2-Nepali-Repaired-Tokenizer-32k}, \url{https://huggingface.co/Aarjan/Qwen3-Nepali-Repaired-Tokenizer-32k}, \url{https://huggingface.co/Aarjan/Llama-3.2-Nepali-Extended-Tokenizer-32k}, \url{https://huggingface.co/Aarjan/Qwen3-Nepali-Extended-Tokenizer-32k}.}
\newcommand{\commitA}{\texttt{7713fd2}}\newcommand{\commitB}{\texttt{5430e86}}\newcommand{\commitC}{\texttt{aa8779a}}\newcommand{\commitD}{\texttt{9be7667}}\newcommand{\commitE}{\texttt{660fd50}}\newcommand{\commitF}{\texttt{9fb12e7}}\newcommand{\commitG}{\texttt{c8bc301}}\newcommand{\commitH}{\texttt{4bc56cc}}\newcommand{\tzlabel}{ (IST)}\newcommand{\commitI}{\texttt{6dddbeb}}

\begin{document}

\title{Merge-Based Vocabulary Extension Stalls at the Pre-Tokenizer: An English-Preserving Repair for Nepali, Tested in Continued Pretraining}

\author{
\IEEEauthorblockN{Aarjan Chaudhary}
\IEEEauthorblockA{\textit{Naamche Labs}}
}

\maketitle

\begin{abstract}
Llama-3, Qwen3 and many other open LLMs pre-tokenize text with a regular expression whose word class admits only Unicode letters. Devanagari vowel signs are combining marks, so this regex splits Nepali words at every vowel sign, and models that use it spend $2.6$ to $5.8$ times as many tokens on Nepali as on the same English text. Merge-based vocabulary extension cannot join the split pieces: for Llama-3 and Qwen3 no number of new merges brings this premium below about $2.2$, and ten released extensions already end within 4\% of this pre-token floor. We give a small edit to the deployed regex that provably keeps every token id of text without combining marks or Devanagari characters. With 32{,}000 new merges it brings Nepali to $1.01$ times English, about 54\% fewer Nepali tokens than standard extension. In small-budget continued pretraining of Llama-3.2-1B and Qwen3-0.6B, the original tokenizer gave the lowest held-out Nepali bits per byte under every recipe we tested. Matching the repair's training steps to standard extension, with about 35\% of the data repeated, reversed their order on Qwen but not on Llama. Whether the token savings give a favourable latency and quality trade-off in generation remains untested. We release the tokenizers and code.
\end{abstract}

\section{Introduction}

On parallel FLORES-200 text, Qwen3's tokenizer needs $4.42$ times as many tokens for a Nepali sentence as for its English translation, and the \texttt{cl100k} tokenizer used by Phi-4, OLMo-2 and Granite-4.1 needs $4.76$ times as many (\S\ref{sec:audit}). We call this ratio the \emph{premium}. Nepali speakers pay it on every request, in price, latency and usable context.

Much of the premium has a mechanical cause. GPT-2-style pre-tokenizers split text with a regular expression whose word class \pL{} matches one or more Unicode letters~\citep{gpt2encoder}; we call it the \emph{letters-only regex}. In Unicode, Devanagari vowel signs, the virama, the anusvara and the candrabindu are combining marks (\pM{}), so the regex starts a new pre-token at each of them and byte-pair encoding (BPE) can never merge across that boundary. \citet{regmi2026vowel} formalised this cause and the resulting bound on token counts. They showed that adding marks to the word class cuts Nepali from 4.78 to 1.58 tokens per word, and that 63\% of the 1{,}345 text-generation repositories they classified still use the letters-only class.

Regmi et al.\ studied tokenizers trained from scratch. Most practitioners instead adapt a model that already exists, usually by vocabulary extension: they add target-language tokens, initialise their embeddings and continue pretraining~\citep{tejaswi-etal-2024-exploring,yamaguchi2026lowres,purason-etal-2026-teaching,smith2026inplace}. We study that deployed-model setting. We show that the bound becomes a floor that released merge-based extensions already hit; we give a repair of a deployed model's regex that provably keeps English token ids; and we test it in controlled continued pretraining against the standard continued-BPE recipe of \citet{purason-etal-2026-teaching}, with the tokenizers released.

\paragraph{Findings.} For Llama-3 and Qwen3 the \emph{pre-token floor}, the lowest premium any merge-based extension can reach, is about $2.2$ (\S\ref{sec:floor}). Ten released extensions for six scripts end within 4\% of their floors, and our own extension under the letters-only regex stalls at it. Registering our repair's learned strings as added tokens~\citep[cf.][]{churchill2026reducing} gets below the floor, but a fifth of the matches then start mid-word. Our repair adds marks to the word class with two string substitutions and, with new merges, brings Nepali to parity with English while leaving every token id of mark-free, non-Devanagari text unchanged (\S\ref{sec:method}). The model-quality result is less favourable. Continued pretraining with no new tokens gave the best held-out Nepali bits per byte on both models, so our main hypothesis failed. At equal bytes the repair is worse than standard extension ($0.025$ bits per byte on Llama and $0.012$ on Qwen, two-seed means), but it takes 26\% fewer steps; at matched steps it is $0.010$ better on Qwen and $0.012$ worse on Llama, and a higher learning rate on Qwen improves all three arms' Nepali scores without changing their order, at a cost on English (\S\ref{sec:cpt}). For the same Nepali text the repair needs 54\% fewer tokens than standard extension; whether this speeds up generation of equally good text is untested.

\paragraph{Contributions.} We make four contributions:
\begin{enumerate}
  \setlength\itemsep{0pt}
  \item Evidence that the pre-token floor binds released merge-based extensions, and a test of the added-token route around it (\S\ref{sec:floor}); an audit of 24 distinct deployed tokenizers with per-tokenizer floors (\S\ref{sec:audit}).
  \item A repair of a deployed regex that provably keeps English token ids (not model behaviour), a Devanagari-only variant with a stronger guarantee (not trained), and a practical pitfall in loading repaired tokenizers (\S\ref{sec:method}).
  \item A controlled comparison of \RO{} (no extension), \RI{} (standard extension) and \RII{} (repair + extension) in continued pretraining of two deployed models, reported against hypotheses written down in advance, including the one that failed (\S\ref{sec:cpt}).
  \item The tokenizers, code and analysis plan.\footnote{\artifactnote}
\end{enumerate}

\section{Related work}
\label{sec:related}

\paragraph{Tokenization premiums.} Tokenizers trained mostly on English encode other languages in more tokens, so those languages pay more for the same content~\citep{rust-etal-2021-good,petrov2023unfairness,ahia-etal-2023-languages,lundin-etal-2026-token}. For Nepali, \citet{kadariya2026nepali} reports $4.79$ for \texttt{cl100k}, and \citet{srivastava2026tokenizertax} finds that \texttt{o200k} cuts the mean Indic premium from $8.0$ to $2.1$ without attributing the change to the regex. We measure sentence-level token ratios on parallel FLORES-200 text~\citep{goyal-etal-2022-flores,nllb2022}, following \citet{petrov2023unfairness}. Fewer tokens do not guarantee a better model~\citep{schmidt-etal-2024-tokenization,zouhar-etal-2023-tokenization,luitel2025perplexity}, so we also test the repair with continued pretraining.

\paragraph{Pre-tokenization of combining marks.} A 2024 tiktoken issue reported that GPT-2-style regexes split Indic words at vowel signs and proposed the class \texttt{[\textbackslash p\{L\}\textbackslash p\{M\}]+}~\citep{tiktoken_issue292}. ZeTT~\citep{minixhofer2024zett} adjusted its regex to avoid over-segmenting Hindi and Tamil, \citet{velayuthan-sarveswaran-2025-egalitarian} showed that the GPT-2, GPT-4 and Llama-3 pre-tokenizers break Tamil, Sinhala and Hindi, and \citet{rana2025mutant} cut Hindi fertility from 3.47 to 1.36 tokens per word by swapping the GPT-2 regex for Llama-4's; see also \citet{thakur2025artofbreaking}. \citet{regmi2026vowel} formalised the cause, stated the bound across 26 languages, and showed with 268M-parameter models trained from scratch that the wider class lowers Nepali bits per byte by 4.4\%. Qwen3.5's released tokenizer uses a mark-aware class~\citep{qwen35tokenizer}, a change also noted by \citet{land2026smallbugs}, and \citet{shrestha2026papaya} study morphology-guided pre-tokenization for Nepali in models trained from scratch. The edit itself is therefore known. What has not been studied is applying it to an already trained model while keeping its English tokenization.

\paragraph{Adapting the vocabulary of a pretrained model.} Vocabulary extension adds tokens, initialises their embeddings from existing ones~\citep{gee-etal-2022-fast,minixhofer-etal-2022-wechsel,dobler-de-melo-2023-focus,liu-etal-2024-ofa,remy2024transtokenization} and continues pretraining~\citep{cui2023chinesellama,tejaswi-etal-2024-exploring,kim2024eeve,csaki2024sambalingo,nag2024efficientcpt}; \citet{duwal2024nepalidapt} continue pretraining Llama-3-8B on Nepali with 4-bit QLoRA and without new tokens. Continued BPE appends new merges to a byte-level model's merge list; \citet{purason-etal-2026-teaching} introduce it and report tokenizer-only compression for many languages, Nepali included, under the unchanged regex. \citet{yamaguchi2026lowres} and \citet{yamaguchi2025elchat} extend Llama-3 and Qwen models for several abugidas, \citet{smith2026inplace} expand LFM2 and \citet{mehta2026defragmenting} Qwen3. \citet{didenko2026writing} adapts models to Ukrainian at fixed vocabulary size, keeping English token counts within 0.05\%. All of these keep the original pre-tokenizer. Tokenizer transfer~\citep{minixhofer2024zett,minixhofer2025alm,dobler2024tokenizerswap}, dynamic tokenization~\citep{feher2024dynamic} and grapheme- or akshara-aware pre-tokenizers~\citep{kapila2026typedriven,darshana2026wwho} replace more of the tokenizer. \citet{dagan2024getting} replace the tokenizer of a pretrained model, regex included, during code fine-tuning and find it recovers after about 50B tokens in their setting; \citet{rana2025mutant} likewise give Llama-3.2-1B a new tokenizer with Llama-4's mark-aware regex and train its embedding and output layers on 40B tokens. Both alter English tokenization; ours does not.

The closest prior system is TituLLM~\citep{nahin-etal-2025-titullms}, which extends Llama-3.2 for Bengali; its released v2.0 tokenizer uses a \texttt{\textbackslash w}-based regex that keeps vowel signs attached. It does not identify the regex as a cause and has no regex-only or vocabulary-only arm, and its regex also changes digit and whitespace pre-tokens in English. We add a controlled attribution on deployed models and a repair that leaves every input without marks or Devanagari characters unchanged.

\section{Pre-tokens bound merge-based extension}
\label{sec:floor}

A byte-level BPE tokenizer~\citep{sennrich-etal-2016-neural,radford2019language} first splits a string $s$ with a regex into pre-tokens $P(s)$ and then encodes each pre-token on its own by applying merges in rank order. Every pre-token yields at least one token, so $|T(s)| \ge |P(s)|$; \citet{regmi2026vowel} formalise this as a training-free lower bound on fertility. Appending merges, whatever they are and however many, changes neither the normalizer nor the regex, so the bound holds for every tokenizer extended this way.

We turn this into a bound on the premium. For a tokenizer $T$ with pre-tokenizer $P$, the \emph{pre-token floor} is $|P(\text{ne})|/|T(\text{en})|$, the number of Nepali pre-tokens divided by the number of English tokens, summed over the 1{,}012 FLORES-200 devtest sentences. A merge-based extension that keeps $P$ and leaves English unchanged cannot bring the premium below this value, though it may stop above it. Our standard extension \RI{} leaves English unchanged: every new merge produces a token that contains a complete Devanagari code point, and since UTF-8 is self-synchronising, no new merge can fire on text without one (the same argument appears in the proof of Proposition~\ref{prop:lossless}, Appendix~\ref{app:proof}).

For today's regexes the floor is far from slack. In Llama-3 and Qwen the word branch is \texttt{[\textasciicircum\textbackslash r\textbackslash n\textbackslash p\{L\}\textbackslash p\{N\}]?\textbackslash p\{L\}+}, a run of letters optionally preceded by one other character. A Devanagari vowel sign ends the current run and becomes the leading character of the next one, so \emph{nep\=alko sark\=ar} (``government of Nepal'') becomes six pre-tokens, four of which start with a vowel sign that belongs to the previous consonant (Figure~\ref{fig:split}, Appendix~\ref{app:extras}). The floor is $2.19$ for Llama-3 and $2.17$ for Qwen3. That is above the premium of \texttt{o200k} ($1.61$), and 2.7 to 2.8 times the floor under the repaired regex of \S\ref{sec:method} ($0.79$ and $0.80$). Llama-3's premium is already $2.58$, so standard extension can lower it by at most 15\%.

\paragraph{Released extensions end at the floor.} Figure~\ref{fig:released} measures released merge-based extensions of letters-only models; Appendix~\ref{app:released} lists every number and how we found them. Ten extensions by Yamaguchi et al.\ of Llama-3, Llama-3.1, Qwen2.5 and Qwen3~\citep{yamaguchi2026lowres,yamaguchi2025elchat} cover six scripts written with combining marks. All ten end within 4\% of their floor, having closed at least 98\% of the distance from the base tokenizer.\footnote{These extensions re-order the base merges and change English token counts slightly (+6.5\% for the Llama-3 Telugu release), so for them we compare tokens with pre-tokens of the target text directly.} The 61k-token in-place expansion of LFM2-8B-A1B into LFM2.5-8B-A1B~\citep{smith2026inplace} ends further from its floor ($1.07$ times it for Hindi, $1.16$ for Nepali, $1.24$ for Bengali and $1.82$ for Thai, which is written without spaces), but its floors are just as high. Under the repaired regex the same text has 2.2 to 5.9 times fewer pre-tokens, so each of these extensions stopped well short of what new vocabulary could achieve. Yamaguchi et al.'s Amharic extension, for a script without combining marks, stays at $1.57$ times its floor, and the repair does not change that floor.

\begin{figure}[t]
\centering
\includegraphics[width=0.95\columnwidth]{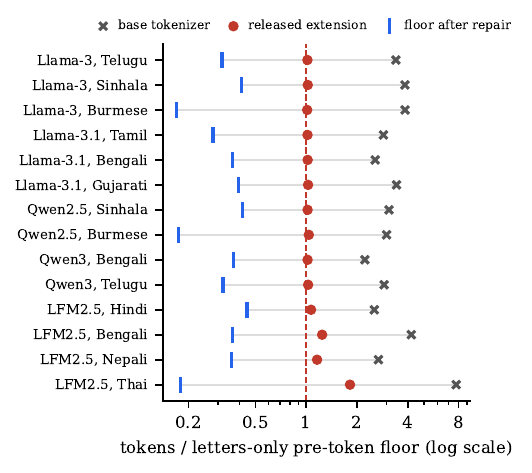}
\caption{Released merge-based extensions of letters-only models on FLORES-200 devtest of each target language, relative to the letters-only pre-token floor. The ten extensions by Yamaguchi et al.\ (top) move from the base tokenizer ($\times$) to within 4\% of the floor (\textcolor{red}{$\bullet$}). LFM2.5 (bottom) stops further above it in all four languages. The blue tick is the floor under the repaired regex. Full numbers, including the Amharic control, are in Table~\ref{tab:released}.}
\label{fig:released}
\end{figure}

\paragraph{Added tokens bypass the floor.} Strings registered as Hugging Face \emph{added tokens} are matched in the raw text before the regex runs, so the bound does not apply to them. SinLlama~\citep{aravinda2025sinllama} extends Llama-3 for Sinhala this way and goes below its letters-only floor. To test this route fairly, we registered as added tokens the 31{,}753 whole Devanagari strings among the 32{,}000 learned for \RII{} (\S\ref{sec:method}), on the unchanged regex. Nepali then costs $1.025$ times English on Llama-3 and $1.009$ on Qwen3, close to \RII{}'s $1.014$, and no FLORES-200 sentence without a Devanagari character changes its token ids. The floor therefore binds merge-based extension only. The price is the segmentation. Added tokens match the longest string anywhere, ignoring word boundaries, and 21\% of the matches start inside a word and 28\% end inside one. Allowing only whole-word matches (\texttt{single\_word=True}) removes those splits and also the gain: the premium returns to $2.49$ on Llama-3 and $4.22$ on Qwen3 (Appendix~\ref{app:extras}). We did not train models with added tokens.

\section{Audit of deployed tokenizers}
\label{sec:audit}

We measured 33 released tokenizers on FLORES-200 devtest. Nine produce token ids identical to another tokenizer's on every sentence of five languages we checked (for example, Phi-4, OLMo-2 and Granite-4.1 match \texttt{cl100k}, and Qwen3 matches Qwen2.5), which leaves 24 distinct tokenizers (Figure~\ref{fig:audit}, Appendix~\ref{app:audit}). We classify each by its own pre-tokenizer configuration, since models in one family do not always share a regex. Four SentencePiece and WordPiece tokenizers (XLM-R, mT5, NLLB-200, IndicBERTv2) do not round-trip their input, so their low premiums are not comparable; we mark them and exclude them from comparisons.

\begin{figure}[t]
\centering
\includegraphics[width=0.95\columnwidth]{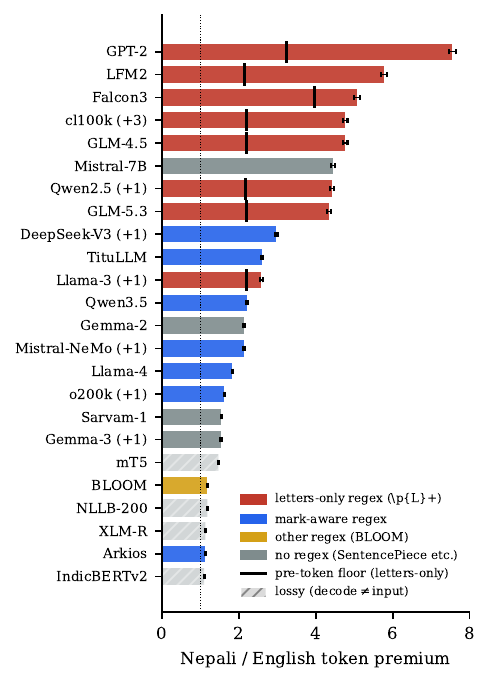}
\caption{Nepali/English premium of 24 distinct tokenizers (33 released; ``+$n$'' counts aliases) on FLORES-200 devtest, with 95\% paired-bootstrap intervals. Black ticks mark the pre-token floor of each letters-only tokenizer: no merge-based extension can move its bar to the left of the tick. Dotted line: parity with English.}
\label{fig:audit}
\end{figure}

Eight distinct tokenizers (13 released models, including GLM-5.3 and LFM2) use the letters-only class, and none of them can be extended below $2.15$. Leaving aside GPT-2 ($7.55$), which we include only for reference, their premiums range from $2.58$ (Llama-3) to $5.77$ (LFM2). The seven lossless tokenizers with a mark-aware regex range from $1.12$ (Arkios~\citep{regmi2026arkios}) to $2.97$ (DeepSeek-V3).

We split each premium into three factors (Appendix~\ref{app:decomp}). The base term, $0.77$ to $0.80$ for every tokenizer, is what the premium would be if every Nepali pre-token under the repaired regex were a single token. The regex factor, $2.77$ to $2.79$ for the Llama-3, Qwen and \texttt{cl100k} pattern, counts the extra pre-tokens the tokenizer's own regex creates. The vocabulary factor, tokens per pre-token, is the only one that merge-based extension can reduce, and it cannot go below 1. For Llama-3 it is already $1.18$, so most of the remaining premium sits in the regex factor. The factors are an accounting identity. The regex factor measures room for improvement, and repairing the regex alone does not lower the premium (\S\ref{sec:compression}). Among mark-aware tokenizers the vocabulary factor still ranges from $1.41$ (Arkios) to $3.71$ (DeepSeek-V3), so a good regex also needs Nepali merges. Within families, the Nepali premium halves from Qwen3 to Qwen3.5 ($4.42$ to $2.20$), whose word branch is close to our repaired class~\citep{qwen35tokenizer}, and falls from $2.58$ to $1.83$ from Llama-3 to Llama-4, which adopts \texttt{o200k}'s pattern. Both vocabularies also grew, so these changes are consistent with our account without isolating the regex.

A well-formed Devanagari syllable never starts with a combining mark, yet 55\% of Llama-3's Nepali tokens do. Even \texttt{o200k}'s mark-aware regex leaves 26\%, because BPE also splits inside words.

\section{Repairing a deployed tokenizer}
\label{sec:method}

The repair has three parts. We edit the regex, learn new merges on top of the existing ones, and initialise the new embeddings so that the model behaves exactly as before on old tokens.

\paragraph{Regex edit.} We add \pM{} to the word class and exclude it from the optional leading character and from the punctuation class. A mark can then only be matched as part of a word, so it stays with the letters around it. For the Llama-3 and Qwen pattern this takes two string substitutions that edit three character classes (Appendix~\ref{app:regex}); digits, contractions, whitespace and punctuation are handled as before. We call the result the \emph{repaired regex}. Qwen3.5 ships almost the same word branch, as its released tokenizer file shows~\citep{qwen35tokenizer}.

\paragraph{Continued BPE.} Following \citet{purason-etal-2026-teaching}, we pre-tokenize Nepali text with the arm's regex, segment every pre-token with the original merges, as the deployed model sees it, and run BPE over these segmentations with each base token as an indivisible symbol. The $K$ new merges are appended after the existing ones, so the original merges act first. We discard any merge whose output contains no complete code point in U+0900--U+097F, together with merges that depend on it. \RI{} and \RII{} run this procedure on the same text with the same $K$ and differ only in the regex.

\begin{proposition}[English invariance]
\label{prop:lossless}
Let $s$ be a normalised string that contains no combining mark (\pM{}) and no code point in U+0900--U+097F. Then the repaired tokenizer \RII{} and the original tokenizer produce identical token ids for $s$.
\end{proposition}
The proof is in Appendix~\ref{app:proof}. The proposition covers ASCII English and code and every script whose normalised text has no combining marks. It does not cover other scripts that use marks: the general repair changes the pre-tokens of 61 of the 204 FLORES-200 languages, including Arabic, Bengali and Thai, whose tokenization the base model has seen only in its original form.

\paragraph{A Devanagari-only repair.} Admitting only the 34 combining marks of the Devanagari block, instead of all of \pM{}, gives a stronger guarantee by the same argument: token ids are unchanged for every input without a code point in U+0900--U+097F. On FLORES-200 this variant changes pre-tokens only in the ten Devanagari languages, and on FLORES Nepali its \RII{} token ids are identical to the general repair's. On the FineWeb-2 Nepali test split (275\,MB, one document per line), 24 of 49{,}462 lines are pre-tokenized differently, all because of marks from other blocks, such as emoji variation selectors and Kannada vowel signs (Appendix~\ref{app:regex}). We recommend this variant for deployment. All experiments in this paper, and the released tokenizers, use the general repair.

\paragraph{Embedding initialisation.} Each new token's input and output rows are set to the mean of the rows of its constituent base tokens~\citep{gee-etal-2022-fast}, obtained by applying the original merges to its byte-level string. Original rows are untouched, so on any input whose ids are all original the resized model gives the same logits for every original token as the base model; on an English probe the largest difference is below $10^{-4}$. The new tokens do add terms to the softmax denominator, and continued pretraining is needed to train them.

\paragraph{Loading pitfall.} In transformers 5.x, the class \texttt{Qwen2Tokenizer} rebuilds the pre-tokenizer from a hard-coded pattern when it loads, which silently undoes a repaired regex stored in \texttt{tokenizer.json}. Our released tokenizers therefore set \texttt{tokenizer\_class} to \texttt{PreTrainedTokenizerFast}, and we check after loading that the premium matches the paper.

\paragraph{Arms.} \RO{} (no extension) keeps the original tokenizer. \RI{} (standard extension) appends $K$ merges learned under the original regex, as in \citet{purason-etal-2026-teaching} and \citet{yamaguchi2026lowres}. \RII{} (repair + extension) switches to the repaired regex and appends $K$ merges learned under it.

\section{Experimental setup}
\label{sec:setup}

We committed an analysis plan to our own repository before training any model, then seven amendments: budget cuts and the primary analysis (before any result), warm-up runs (after the first Llama results), a Llama rerun with the beginning-of-sequence (BOS) token in training, matched-steps and second-seed runs, further seeds with a learning-rate check, \RO{} at the higher rate, and a missing second seed (each of the last five before running it). The study up to the BOS rerun took about 11 hours, and several analysis choices changed after we saw results; Appendix~\ref{app:prereg} lists every commit and deviation, and Appendix~\ref{app:protocol} the training and evaluation details.

\paragraph{Models and data.} We use Llama-3.2-1B (base)~\citep{grattafiori2024llama3} and Qwen3-0.6B-Base~\citep{yang2025qwen3}, whose tokenizers ship the letters-only regex and which tie input and output embeddings. The Qwen3 tokenizer is identical across model sizes; we built it on Qwen3-1.7B-Base. Nepali text is FineWeb-2 \texttt{npi\_Deva}~\citep{penedo2025fineweb2} and English is FineWeb-Edu~\citep{penedo2024fineweb}. We normalise to NFC, deduplicate, and remove every document that shares a whitespace 10-gram with FLORES-200 dev or devtest. A 30\,MB held-out slice of each language is never trained on. \RI{} and \RII{} learn merges on the same 1\,GB of Nepali, with the registered $K{=}32{,}000$.

\paragraph{Continued pretraining.} Each arm is trained for one epoch on the same documents, 1\,GB of Nepali and 1\,GB of English interleaved in one fixed order, with identical optimisation: a fixed, untuned peak learning rate of $10^{-4}$ and cosine decay (later amendments add second seeds and, on Qwen, a rate of $3{\times}10^{-4}$). Because \RII{} encodes the same bytes in fewer tokens, it takes fewer steps. \RO{} and \RI{} also save a checkpoint at the step where \RII{} ends, before their learning rate has fully decayed. We call these the \emph{main runs}. The \emph{warm-up runs} (Llama only) repeat them after first training only the embedding matrix on 10\% of the stream at a learning rate of $10^{-3}$.

\paragraph{Evaluation.} All metrics are normalised by UTF-8 bytes. The primary metric is bits per byte (BPB) on held-out native Nepali, \emph{byte-matched}: each document is split at whitespace into pieces of about 2{,}000 bytes, and every arm scores the same pieces. Each piece is scored after one context token. The \emph{registered prefix} is the BOS token for Llama and the end-of-text token for Qwen, whose tokenizer defines no BOS token. In the main and warm-up runs, however, documents were separated only by the end-of-text token, so the Llama models never saw BOS. We therefore also score them after that token, the \emph{training separator}, as a diagnostic chosen after seeing results, and we reran the Llama main runs with BOS at every document start (\S\ref{sec:cpt}). Secondary metrics (held-out English BPB, FLORES BPB, Belebele, 5-shot translation) are in Appendix~\ref{app:cptfull}.

\paragraph{Analysis.} The registered main hypothesis was that \RII{} is non-inferior to both \RI{} and \RO{} on held-out Nepali BPB, with a margin of 0.01, while using fewer training tokens. The labels for the 95\% interval of (\RII{} $-$ other) are applied in a fixed order: \emph{better} if it lies entirely below 0; otherwise \emph{non-inferior} if it lies entirely below $+0.01$; otherwise \emph{worse} if it lies entirely above 0; otherwise \emph{no detectable difference}. So $[+0.004, +0.007]$ is non-inferior although it excludes zero. The label ``worse'' was added in the second amendment. The registered test was a paired bootstrap over held-out pieces. Pieces of one document are correlated, so we report a bootstrap over documents instead ($B{=}10{,}000$). Clustering by the 503 held-out documents (Appendix~\ref{app:protocol}) widens the Nepali intervals 1.4 to 2.6 times and changes no verdict. The same test on English pieces is the registered English control. All intervals describe the sampling of evaluation text. Run-to-run variance is measured only by second seeds (Appendix~\ref{app:extra-runs}).

\section{Results}
\label{sec:results}

\begin{table*}[t]
\centering\small
\setlength{\tabcolsep}{4.5pt}
\begin{tabular}{lccccc}
\toprule
 & \multicolumn{4}{c}{Llama-3.2-1B} & Qwen3-0.6B \\
\cmidrule(lr){2-5}\cmidrule(lr){6-6}
 & \multicolumn{2}{c}{main runs} & warm-up runs & BOS rerun & main runs \\
\cmidrule(lr){2-3}\cmidrule(lr){4-4}\cmidrule(lr){5-5}\cmidrule(lr){6-6}
 & registered prefix & training sep. & training sep. & registered prefix & (reg.\ = sep.) \\
\midrule
\multicolumn{6}{l}{\emph{Held-out Nepali BPB} (lower is better)} \\
Base model, no CPT & 0.766 & n/a & n/a & 0.766 & 0.852 \\
\RO{} & 0.474 & 0.471 & 0.471 & 0.471 & 0.501 \\
\RI{} & 0.534 & 0.483 & 0.480 & 0.483 & 0.561 \\
\RII{} & 0.539 & 0.510 & 0.501 & 0.509 & 0.570 \\
\midrule
\multicolumn{6}{l}{\emph{Paired difference in held-out Nepali BPB} (positive: first arm worse), document-clustered 95\% interval} \\
\RI{} $-$ \RO{} & +0.060 & +0.012 & +0.010 & +0.012 & +0.060 \\
 & [+0.058, +0.061] & [+0.011, +0.013] & [+0.009, +0.011] & [+0.011, +0.013] & [+0.058, +0.062] \\[2pt]
\RII{} $-$ \RO{} & +0.066 & +0.039 & +0.031 & +0.038 & +0.070 \\
 & [+0.064, +0.068] & [+0.038, +0.041] & [+0.029, +0.032] & [+0.037, +0.039] & [+0.067, +0.072] \\[2pt]
\RII{} $-$ \RI{} & +0.006 & +0.027 & +0.021 & +0.026 & +0.009 \\
 & [+0.004, +0.007] & [+0.026, +0.028] & [+0.020, +0.022] & [+0.024, +0.027] & [+0.008, +0.010] \\
\midrule
Training tokens (M) & \multicolumn{3}{c}{417 / 386 / 286} & 418 / 386 / 286 & 574 / 392 / 292 \\
Decoding steps & \multicolumn{4}{c}{2.55 / 2.17 / 1} & 4.35 / 2.16 / 1 \\
\bottomrule
\end{tabular}
\caption{Continued pretraining, seed 0 of each arm, on the same 1\,GB of Nepali and 1\,GB of English. \textbf{Our results are the BOS-rerun (Llama) and Qwen columns}; two-seed means and matched-steps runs are in Table~\ref{tab:extra-runs}; the other Llama columns show the scoring artifact of \S\ref{sec:cpt}. BPB: held-out native Nepali, in pieces of about 2{,}000 bytes, scored after the registered prefix (BOS for Llama; end-of-text for Qwen, which has no BOS) or the training separator (end-of-text; a post hoc diagnostic). Intervals resample 503 held-out documents and reflect only the sampling of evaluation text. Training tokens and decoding steps (FLORES Nepali, relative to \RII{}) are \RO{} / \RI{} / \RII{}. Secondary metrics: Table~\ref{tab:cpt}.}
\label{tab:prefix}
\end{table*}

\subsection{Compression}
\label{sec:compression}

Figure~\ref{fig:ksweep} sweeps the number of new merges $K$. Under the letters-only regex (\RI{}), extension saturates almost at once. For Llama-3 the first 1{,}000 merges take the premium from $2.58$ to $2.32$, and the next 63{,}000 only to $2.20$, within $0.01$ of the floor. For Qwen3, \RI{} reaches $2.19$ by $K{=}16$k and stops at $2.18$, against a floor of $2.17$. It stops early because the letters-only regex leaves few distinct Nepali pre-tokens to merge (Appendix~\ref{app:extras}). Under the repaired regex (\RII{}), additional merges keep helping. At $K{=}32$k Nepali reaches parity with English ($1.01$) on both models, and at $64$k it is cheaper than English ($0.95$), below \texttt{o200k}. At $K{=}32$k, \RII{} needs 46\% of \RI{}'s tokens for FLORES Nepali. All 28 extended tokenizers decode losslessly and reproduce the original token ids on all 1{,}012 English devtest sentences, as \S\ref{sec:floor} and Proposition~\ref{prop:lossless} predict.

\begin{figure}[t]
\centering
\includegraphics[width=\columnwidth]{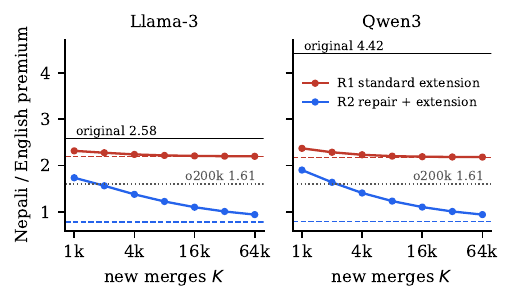}
\caption{Nepali/English premium on FLORES-200 devtest after adding $K$ merges by continued BPE, under the letters-only regex (\RI{}) and the repaired regex (\RII{}). Dashed lines: pre-token floors under each model's own pattern ($2.19$ and $2.17$ letters-only; $0.79$ and $0.80$ repaired, for Llama-3 and Qwen3). 95\% paired-bootstrap intervals (half-width at most 0.03) are narrower than the markers.}
\label{fig:ksweep}
\end{figure}

Both parts of the repair are needed. Repairing the regex without new merges changes almost nothing ($2.61$ for Llama-3, $4.42$ for Qwen3), because the old merges still cut Nepali words into the old pieces. The repair also changes where tokens start. The share of Nepali tokens that start with a combining mark, a rough sign of a split syllable, rises from 55\% (Llama-3) and 32\% (Qwen3) to 65\% and 64\% under \RI{}, since new merges can only build on mark-initial fragments, and falls to 13\% and 11\% under \RII{}. The mark-initial share does not predict model quality in \S\ref{sec:cpt}: \RII{} has by far the fewest mark-initial tokens yet trails \RO{} on both models.

We repeated the sweep for seven other languages written with combining marks and two controls (Table~\ref{tab:xling}, Appendix~\ref{app:xling}), learning merges on 25 to 361\,MB of FineWeb-2 text per language. At $K{=}32$k, \RII{} brings all seven to between $0.88$ and $1.24$ times English. \RI{} ends within 5\% of its floor in six of them. Thai is the exception: it is written without spaces, so its letters-only pre-tokens are long and leave \RI{} room. For Amharic and Korean, whose scripts have no combining marks, \RI{} and \RII{} are identical at every $K$, as the mechanism predicts.

\subsection{Continued pretraining}
\label{sec:cpt}

\paragraph{The registered outcome.} Table~\ref{tab:prefix} gives the results. The main hypothesis fails against \RO{}: continued pretraining without new tokens gave the lowest held-out Nepali BPB on both models. On Qwen, \RII{} was worse than \RO{} by $0.070$, far beyond the 0.01 margin, and \RI{} was worse by $0.060$. Against standard extension, \RII{} $-$ \RI{} $= +0.0094$ $[+0.0084, +0.0104]$ (registered piece-level bootstrap: $[+0.0086, +0.0101]$), labelled ``worse'' although the point estimate lies inside the margin; averaged over two seeds it is $+0.0122$ $[+0.0112, +0.0132]$.

\paragraph{The BOS-consistent rerun.} The main Llama runs never saw BOS (\S\ref{sec:setup}). Under the registered prefix they give \RII{} $-$ \RI{} $= +0.0058$ $[+0.0043, +0.0073]$, non-inferior although it excludes zero, but $+0.027$ after the training separator. A third amendment retrained the three Llama arms with BOS before every document and nothing else changed; we take this rerun as our Llama result. \RO{}, \RI{} and \RII{} reach $0.471$, $0.483$ and $0.509$: \RI{} is worse than \RO{} by $0.0124$ $[0.0114, 0.0135]$ and \RII{} worse than \RI{} by $0.0256$ $[0.0245, 0.0267]$, about 5\% ($0.0252$ over two seeds). These match the separator diagnostic within $0.002$, so the main runs' registered-prefix numbers reflect a scoring artifact of the unseen BOS token; after the end-of-text token the rerun scores $0.875$ to $1.666$ (Appendix~\ref{app:cptfull}).

\paragraph{Matched steps.} At equal bytes, \RII{} takes 26\% fewer steps than \RI{}. A fourth amendment retrained \RII{} for \RI{}'s step count (737 for Llama with BOS, 748 for Qwen), drawing the extra 35\% of windows from a second pass over its stream, and a fifth added second seeds. Averaged over two seeds, matched-steps \RII{} is \emph{better} than \RI{} on Qwen by $0.0099$ $[0.0088, 0.0109]$ ($0.549$ against $0.559$) and worse on Llama by $0.0120$ $[0.0111, 0.0129]$ ($0.495$ against $0.483$); \RO{} stays best, by $0.049$ and $0.024$. Matching steps thus reverses the equal-bytes gap on Qwen and halves it on Llama, though with 35\% of the data repeated. Seeds move each arm by about $0.003$ or less (Appendix~\ref{app:extra-runs}).

\paragraph{Learning rate.} A fifth and a sixth amendment repeated the Qwen arms at $3{\times}10^{-4}$ instead of $10^{-4}$ (one seed, \RII{} at matched steps). Every arm improves, \RO{} by $0.021$, \RI{} by $0.052$ and \RII{} by $0.046$, but the order holds: \RO{} $0.479$, \RII{} $0.504$, \RI{} $0.509$. \RII{} beats \RI{} by $0.0057$ $[0.0047, 0.0066]$ and trails \RO{} by $0.024$ $[0.023, 0.026]$, so our predictions that \RO{} would gain little and \RII{} would match it failed. English gets worse: $0.875$ to $0.879$ BPB, against $0.862$ to $0.864$ at $10^{-4}$ and $0.873$ for the base model.

\paragraph{What the repair buys.} The clear gain is in token count. For the same FLORES Nepali text, \RII{} needs 46\% of \RI{}'s tokens on both models, and 39\% (Llama) and 23\% (Qwen) of \RO{}'s. We have not shown faster generation of equally good text: our timings forced 128 generated tokens per model on shared GPUs, used the original Llama runs and compared different outputs. In them \RII{}'s generations carry 1.55 (Llama) and 1.47 (Qwen) times as many bytes per token as \RI{}'s, and throughput is 1.48 and 1.50 times higher, which we treat as indicative only.

\paragraph{English.} The registered English control holds: \RII{} $-$ \RO{} on held-out English is $-0.0010$ $[-0.0014, -0.0007]$ (Llama rerun) and $-0.0028$ $[-0.0032, -0.0024]$ (Qwen), and \RII{} beats \RI{} by $0.0022$ on both. ``English-preserving'' refers to token ids only: every Llama arm ends worse on English than the base model ($0.798$ to $0.804$ against $0.788$).

\section{Discussion}
\label{sec:discussion}

\paragraph{Recommendations.} Before extending a letters-only model for a language with combining marks, compare the target text's pre-tokens with its tokens; if they are close, as for Llama-3 on Nepali, merges have little room. Under our recipes, plain continued pretraining gave the lowest held-out Nepali BPB, as \citet{zhao2024llamabeyond} found for Chinese, so extension needs another reason, such as token cost. Then the repair is worth considering: at matched steps it was within about 0.012 BPB of standard extension either way and needs 54\% fewer tokens for the same Nepali text, though we have not shown a latency gain for equally good outputs. We would use the Devanagari-only repair (untrained here; identical to ours on FLORES Nepali), load it with a class that does not rebuild its pre-tokenizer, and score models after the context they were trained with.

\paragraph{Why our sign differs from Regmi et al.} From scratch the wider class lowers Nepali BPB by 4.4\% at equal compute and still leads at equal bytes~\citep{regmi2026vowel}; retrofitted at equal bytes it is worse, and the gap decreases under longer training with repeated data \citep[cf.][]{shrestha2026papaya}: matched steps flip the sign on Qwen and halve the gap on Llama.

\section*{Limitations}

\begin{itemize}
  \setlength\itemsep{0pt}
  \item \textbf{The main hypothesis failed.} At our budget, neither extension recipe beat continued pretraining without new tokens on Nepali BPB. At equal bytes the repair was also worse than standard extension on both models; at matched steps it was better on Qwen and still worse on Llama.
  \item \textbf{Scale.} We adapt 1B- and 0.6B-parameter models with 1\,GB of text per language and one epoch; 1\,GB of Nepali also holds less content than 1\,GB of English~\citep{arnett-etal-2024-bit}. The registered 3\,GB runs were dropped for budget reasons, so we cannot say whether the gap between arms closes with more data; \citet{dagan2024getting} found recovery from a tokenizer change only after about 50B tokens in their code-adaptation setting.
  \item \textbf{Matched steps repeat data.} The matched-steps \RII{} runs reach \RI{}'s step count by seeing 35\% of their stream twice; they hold compute equal, not data, though repeating data for a few epochs costs little~\citep{muennighoff2023scaling}. We did not run \RO{} at matched steps, and Nepali is still a smaller share of \RII{}'s training tokens (26\% against 45\% for \RI{}).
  \item \textbf{Scoring two tokenizations.} For the seed-0 equal-bytes models we add the base tokenization to the canonical one, a partial step towards the marginal likelihood~\citep{cao-rimell-2021-evaluate,chirkova-etal-2023-marginalize}. It changes no score, because the models give the base tokenization almost no probability; this is a weak check, since nearer alternatives, such as splitting one new token, could still carry mass.
  \item \textbf{Few runs.} Every interval reflects only the sampling of evaluation text. In the Llama BOS rerun and on Qwen, every arm at $10^{-4}$ has two seeds, which differ by about $0.003$ or less per arm; the runs at $3{\times}10^{-4}$ have one.
  \item \textbf{Clustering.} Nepali intervals resample the 503 held-out documents; English intervals still cluster pieces by a length heuristic.
  \item \textbf{Learning rate.} We tried two learning rates, on Qwen only and with one seed at the higher one; continued-pretraining schedules matter~\citep{ibrahim2024simple}, and the English loss at the higher rate suggests forgetting. Moving from $10^{-4}$ to $3{\times}10^{-4}$ improved every arm by $0.02$ to $0.05$ without changing their order, so the best rate for each arm may differ from both and Llama was not tuned at all. Mean initialisation may also suit \RII{}'s long new tokens less well than \RI{}'s short ones~\citep{joshi2026beyondinit}.
  \item \textbf{BOS in training.} The main and warm-up Llama runs never saw BOS, which the registered evaluation puts first, so their registered scores and all their secondary metrics carry this artifact; the warm-up translation scores are unusable. The rerun with BOS in training resolves the main comparison, but only on BPB: we did not run Belebele, translation or generation for it, nor a warm-up version. Qwen has no BOS token and was not affected.
  \item \textbf{Weak secondary metrics.} Belebele accuracy is near chance for every model, translation scores come from a single greedy pass without intervals and keep only the first output line, which is empty whenever a model begins its answer with a newline (we did not count how often), and measured generation speed was taken on shared GPUs (Appendix~\ref{app:cptfull}).
  \item \textbf{Equal compute.} The equal-compute checkpoints of \RO{} and \RI{} are taken before their learning rate has fully decayed, and are scored on 1\,MB only.
  \item \textbf{Languages.} We train models for Nepali only. For other languages we report compression, and the general repair, which we trained and released, changes the tokenization of 61 FLORES-200 languages whose quality we did not measure, including Hindi and Thai, which Llama-3.2 officially supports. The cross-script tokenizers use a relaxed merge filter not covered by Proposition~\ref{prop:lossless} (Appendix~\ref{app:xling}).
  \item \textbf{Added tokens and segmentation.} We measured the added-token route around the floor only at the tokenizer level and trained no model with it; \citet{churchill2026reducing} add tokens post hoc to Llama 3.2 1B for whole characters, not words, and \citet{balde-etal-2024-adaptive} change how added tokens are matched. Our objection to it is segmentation, and we have no evidence that segmentation matters for quality. Whether the regex repair beats added tokens is untested.
  \item \textbf{Scoring context.} The BOS-consistent rerun scores every piece after BOS, but 1{,}756 of the 2{,}259 pieces start mid-document, a context these models never saw after BOS in training.
  \item \textbf{What would settle it.} A learning-rate sweep on both models, better embedding initialisation and a curve over data budgets would show whether \RO{}'s lead and the Llama gap persist.
\end{itemize}

\section*{Ethics and use of AI assistance}

All text and evaluation data are public (FineWeb-2, FineWeb-Edu, FLORES-200, Belebele) and used under their licences; we collected no data from people. Released Llama-derived tokenizers follow the Llama 3.2 Community License~\citep{llama32license}; Qwen3-derived tokenizers are Apache-2.0.

We used a large language model assistant extensively. Under our direction it wrote and ran the experiment code, searched the literature and drafted the text. We reviewed and approved the study design, the paper and all citations, and take responsibility for the content. The tokenizer-side numbers can be reproduced from public tokenizers and data with our released code, and every continued-pretraining statistic from the released per-piece scores.

\bibliographystyle{IEEEtranN}
\bibliography{refs}

\appendices
\section{Pre-registration history}
\label{app:prereg}

We committed the analysis plan to the project repository in nine steps between 2026-10-04 and 2026-10-06\tzlabel{}, and report all nine. Times are commit times.

\begin{enumerate}
  \setlength\itemsep{0pt}
  \item \textbf{Original registration} (\commitA{}, 17:59) registered hypotheses about the mechanism, a regex-only tokenizer ablation, a cross-script dose--response and training from scratch. Our literature check then found that \citet{regmi2026vowel} had published the mechanism, the ablation, the cross-script analysis and a from-scratch comparison five weeks earlier. We withdrew those hypotheses as contributions.
  \item \textbf{Retrofit registration} (\commitB{}, 18:16), before any model was trained or evaluated, registered the study reported here: arms \RO{}, \RI{}, \RII{} and the untrained base model; $K{=}32{,}000$ unless \RII{} saturated earlier on FLORES dev; equal-bytes continued pretraining of Llama-3.2-1B and Qwen3-1.7B-Base on a fixed byte budget split equally between Nepali and English (3\,GB each in our plan, a figure first written down in the first amendment); a learning rate selected once per model on \RO{}; two seeds. Its hypotheses were P1 (compression ordering, with \RI{} above the letters-only floor), P2, the main hypothesis (held-out Nepali BPB of \RII{} at most 0.01 above both \RI{} and \RO{}, while using fewer training tokens), P3 (English BPB of \RII{} and \RO{} within 0.01) and P4 (\RII{} decodes Nepali faster at equal bytes). P1 had already been measured when it was written, so we report it as a measurement.
  \item \textbf{First amendment} (\commitC{}, 19:50; file \texttt{PREREG\_AMENDMENT.md}, which numbers the amendments from 2 because it counts the retrofit registration as the first), written after an internal review and before any continued-pretraining run had produced an evaluation, recorded the deviations below, fixed four pipeline bugs, and fixed the primary analysis: byte-matched BPB, a paired bootstrap over pieces and three claim labels (better, non-inferior, no detectable difference; \S\ref{sec:setup} gives the later ordered set). Earlier launches had been stopped by us because of the bugs, and none produced an evaluated checkpoint. The only earlier evaluation, of the untouched base Llama-3.2-1B with a since-fixed prefix bug, is discarded.
  \item \textbf{Second amendment} (\commitD{}, 21:46; \texttt{PREREG\_AMENDMENT3.md}) was written \emph{after} observing the Llama \RO{} and \RII{} main-run results and before any \RI{} or Qwen result. It registered the warm-up runs, with the prediction that they shrink the \RII{} $-$ \RO{} gap, and 3\,GB runs that we did not run, and it added the ``worse'' category to the claim rules.
  \item \textbf{Third amendment} (\commitE{}, 2026-10-05, 03:24; \texttt{PREREG\_AMENDMENT4.md}) was written after all main and warm-up results and before running what it registered: the three Llama main runs repeated with every training document wrapped as BOS, document, end-of-text, scored with the registered protocol. It predicted that scores after BOS and after the end-of-text token would agree within 0.005 for every arm, and that \RII{} $-$ \RI{} after BOS would lie between the two earlier values ($+0.006$ and $+0.027$). The second prediction held ($+0.026$), and the rerun matches the earlier separator diagnostic within 0.002, which confirms that diagnostic in spirit. The first prediction failed: after the end-of-text token the rerun models score far worse (Appendix~\ref{app:cptfull}).
  \item \textbf{Fourth amendment} (\commitF{}, 2026-10-05, 16:08; \texttt{PREREG\_AMENDMENT5.md}) was written after all earlier results and two external reviews, and before running what it registered: \RII{} retrained for \RI{}'s step count on Llama (BOS setting) and Qwen, Qwen \RI{} and \RII{} with a second seed, and two-tokenization scoring. It predicted that matched-steps \RII{} would be non-inferior to \RI{} on Qwen (it was better), smaller-gapped but still worse on Llama (held), still worse than \RO{} on both (held), that a second Qwen seed would move \RII{} $-$ \RI{} by less than 0.005 (it moved by 0.0057, so this failed narrowly), and that two-tokenization scoring would change no verdict (held).
  \item \textbf{Fifth amendment} (\commitG{}, 2026-10-06, 00:48; \texttt{PREREG\_AMENDMENT6.md}), written after the fourth amendment's results and before running, added second seeds for matched-steps \RII{} (both models), Llama \RO{} and \RI{} and Qwen \RO{}, and Qwen \RI{} and matched-steps \RII{} at a learning rate of $3{\times}10^{-4}$, under a budget cap. It predicted that with seeds averaged matched-steps \RII{} would be better than \RI{} on Qwen and worse on Llama, that \RO{} would stay best, and that \RII{} $-$ \RI{} would keep its sign at $3{\times}10^{-4}$. All four held.
  \item \textbf{Sixth amendment} (\commitH{}, 2026-10-06, 05:01; \texttt{PREREG\_AMENDMENT7.md}), written after the fifth amendment's results and before running, added Qwen \RO{} at $3{\times}10^{-4}$. It predicted that \RO{} would gain less than 0.005 and that matched-steps \RII{} would be non-inferior to it. Both failed: \RO{} gained $0.021$ and stays ahead by $0.024$.
  \item \textbf{Seventh amendment} (\commitI{}, 2026-10-06, 07:30; \texttt{PREREG\_AMENDMENT8.md}), written after an external review and before running, added a second seed for equal-bytes Llama \RII{}. It predicted that the two-seed \RII{} $-$ \RI{} would stay worse and within 0.005 of $+0.0256$; it held ($+0.0252$ $[+0.0241, +0.0263]$).
\end{enumerate}
Two departures from these amendments: the fourth registered two-tokenization scoring also for its new runs, but we applied it only to the six seed-0 equal-bytes models of the Llama rerun and Qwen; and the first amendment stated that we would not tune the learning rate on results, whereas the fifth and sixth test a second rate after seeing results. We use that rate only as a robustness check, not to select one.

In the repository these are \texttt{PREREG.md}, \texttt{PREREG\_RETROFIT.md}, \texttt{PREREG\_AMENDMENT.md}, and \texttt{PREREG\_AMENDMENT3.md} to \texttt{PREREG\_AMENDMENT8.md}. The first four commits lie within four hours, the fifth about six hours later, and the last three over the following day. All are in a repository we control, so they document our plan without an external timestamp. Apart from the author line of \texttt{PREREG.md}, no plan file was edited after its commit.

\paragraph{Outcome of each registered hypothesis.} P1 holds as a measurement (\S\ref{sec:compression}). P2 fails against \RO{} on both models. Against \RI{} it fails for Llama in the BOS-consistent rerun (\RII{} is worse by $0.026$), and it also fails for Qwen (\RII{} is worse by $0.009$, $0.012$ over two seeds); it held for Llama only in the main runs under the registered prefix, a result of the BOS artifact (\S\ref{sec:cpt}). P3 holds. P4, faster generation in bytes per second at equal generated bytes, was not reliably evaluated: our timings ran on shared GPUs with a fixed number of generated tokens (Appendix~\ref{app:cptfull}). We report token counts as a separate result. The warm-up prediction holds after the training separator; under the registered prefix the warm-up scores are dominated by the BOS artifact.

\paragraph{Deviations from the retrofit registration.} The first amendment recorded these before any result:
\begin{itemize}
  \setlength\itemsep{0pt}
  \item Qwen3-0.6B-Base replaces Qwen3-1.7B-Base, for budget and time. Their tokenizers are identical.
  \item 1\,GB of Nepali and 1\,GB of English per arm replace 3\,GB each, with identical documents in every arm.
  \item A fixed learning rate of $10^{-4}$ replaces selection on \RO{} from $\{3{\times}10^{-5}, 10^{-4}, 3{\times}10^{-4}\}$ by a 200-step pilot, which was never completed (the candidate set and pilot were first recorded in the first amendment).
  \item One seed per arm replaces two.
  \item Added: \RO{} and \RI{} checkpoints at \RII{}'s final step (equal-compute view); byte-matched BPB as the primary metric, with the registered sliding-window BPB as secondary; the paired bootstrap with claim rules as the primary test; a translation cap of 512 instead of 256 new tokens.
\end{itemize}
Besides the runs added by the second to seventh amendments, we made four further changes after seeing results, all reported as such: intervals from a bootstrap over reconstructed documents instead of pieces (\S\ref{sec:setup}); the training-separator diagnostic; decoding steps instead of measured bytes per second for P4, whose implementation had fixed the number of generated tokens rather than bytes; and the ``worse'' label from the second amendment for the main runs.

\paragraph{Pipeline bugs fixed before any reported run.} (1) Documents had been stored one per line with their internal newlines, so ``documents'' were paragraphs; they are now stored as one JSON-encoded document per line. (2) New-token embeddings had been initialised by decoding tokens to text, which turns partial UTF-8 tokens into U+FFFD; constituents now come from applying the original merges to the byte-level string. (3) The initialisation check had covered Qwen's padding rows; it now covers only the original vocabulary. (4) Evaluation had prefixed documents with the end-of-text token, while base Llama expects BOS; the registered evaluation uses BOS where it exists.

\section{Repair details}
\label{app:regex}

\paragraph{Regex edit.} For the Llama-3 and Qwen pattern we make two string substitutions. The first replaces \texttt{\textbackslash p\{L\}+} by \texttt{[\textbackslash p\{L\}\textbackslash p\{M\}]+} in the word branch. The second inserts \texttt{\textbackslash p\{M\}} after \texttt{\textbackslash p\{L\}} in the two negated classes, the optional leading character of the word branch and the punctuation branch. Three character classes change:
\begin{center}\small
\texttt{[\textasciicircum\textbackslash r\textbackslash n\textbackslash p\{L\}\textbackslash p\{N\}]?\textbackslash p\{L\}+}\\
$\downarrow$\\
\texttt{[\textasciicircum\textbackslash r\textbackslash n\textbackslash p\{L\}\textbackslash p\{M\}\textbackslash p\{N\}]?[\textbackslash p\{L\}\textbackslash p\{M\}]+}\\[4pt]
\texttt{\textvisiblespace?[\textasciicircum\textbackslash s\textbackslash p\{L\}\textbackslash p\{N\}]+}
$\;\rightarrow\;$
\texttt{\textvisiblespace?[\textasciicircum\textbackslash s\textbackslash p\{L\}\textbackslash p\{M\}\textbackslash p\{N\}]+}
\end{center}
Nothing else in the pattern changes.

\paragraph{Devanagari-only variant.} The variant replaces \pM{} in the three classes by an explicit class of the 34 combining marks of the Devanagari block (U+0900--0903, U+093A--093C, U+093E--094F, U+0951--0957, U+0962--0963). We checked that this class equals the set of \pM{} code points in the block under Python's \texttt{unicodedata} (Unicode 15.0), the \texttt{regex} module and the Hugging Face \texttt{tokenizers} engine. Every other character has the same class membership as in the original regex, so any text without a Devanagari mark is pre-tokenized exactly as before. Across the 204 FLORES-200 devtest files, the general repair changes pre-tokens in 61 languages for both models and the Devanagari-only variant in exactly the ten Devanagari languages. On FLORES Nepali the two variants give identical pre-tokens and identical \RII{} token ids. On the FineWeb-2 Nepali test split (275\,MB, 49{,}462 documents, one per line), 24 lines differ, in 47 segments; the marks involved are emoji variation selectors, Kannada, Bengali, Thai and Hebrew signs and one combining candrabindu in romanised text, and the \RII{} token count of those lines is the same under both.

\paragraph{Continued BPE.} At inference BPE applies the lowest-ranked applicable merge first, so the original merges act before the new ones. The extended tokenizer reproduces the training procedure except in two cases: a new merge whose output string is already in the original vocabulary can let an original merge fire again afterwards, and under Hugging Face's \texttt{ignore\_merges} option (set for Llama-3) a pre-token that is itself a vocabulary entry is emitted directly. Neither case affects Proposition~\ref{prop:lossless}. We run merge learning inside the Hugging Face Rust trainer by mapping each base token to a private-use code point; learning $K{=}32{,}000$ merges on 1\,GB of Nepali takes 1.7 to 3.4 minutes per tokenizer.

\paragraph{Scope of Proposition~\ref{prop:lossless}.} The proposition covers ASCII English and code and every script whose normalised text has no combining marks, including most Cyrillic, Han, Hangul and Ethiopic text. It does not cover Devanagari, which changes by design, or text with marks, such as Hindi, Arabic with vowel points, Thai, or Latin written in decomposed form (NFD). The Llama-3 tokenizer has no normalizer, so deployments should apply NFC first; Qwen3's tokenizer applies NFC itself. It also assumes that the regex engine's \pM{} table matches the one used to define $s$; for the Devanagari-only variant we checked this directly. Empirically, \RI{} and \RII{} at $K{=}32{,}000$ give token ids identical to the original tokenizer on all 1{,}012 FLORES devtest sentences in English, Russian and Amharic, for both models, including one Russian sentence with a combining acute that lies outside the proposition.

\section{Proof of Proposition~\ref{prop:lossless}}
\label{app:proof}

\begin{proof}
\emph{Pre-tokens.} The repaired regex differs from the original only in the three classes of Appendix~\ref{app:regex}, and each differs from its original only on \pM{} characters. A regex match on $s$ consults only characters of $s$, and the order of the alternatives is unchanged. Since $s$ has no \pM{} character, both regexes produce the same pre-tokens.
\emph{Merges and vocabulary.} \RII{} differs from the original tokenizer only by appended merges and the vocabulary entries they create, each of which contains a complete code point in U+0900--U+097F. A merge $(a,b)$ applies only where $a$ and $b$ are adjacent symbols of a pre-token, so its output is a contiguous substring of that pre-token's bytes. A vocabulary entry is emitted directly under \texttt{ignore\_merges} only if it equals the whole pre-token. Neither can happen for a pre-token of $s$, because UTF-8 is self-synchronising and $s$ contains no such code point. Each pre-token of $s$ is therefore encoded with original merges and original vocabulary entries only, exactly as by the original tokenizer.
\emph{Ids.} New ids are appended above the original ids, so all ids coincide.
\end{proof}

\section{Training and evaluation protocol}
\label{app:protocol}

\paragraph{Data.} We normalise to NFC, drop documents under 200 characters, remove exact duplicates, and drop every document that shares a whitespace 10-gram with FLORES-200 dev or devtest in Nepali or English. This also covers the Belebele passages, which are FLORES passages~\citep{bandarkar-etal-2024-belebele}, but not its questions or answers. Documents are stored one JSON-encoded document per line. A 30\,MB held-out slice of each language is carved from the cleaned stream before training and never trained on. The 1\,GB of Nepali used to learn merges is taken from the start of the training stream.

\paragraph{Continued pretraining.} Documents are interleaved and separated by the model's end-of-text token; in the main and warm-up runs no BOS token is inserted, and in the Llama BOS-consistent rerun every document is wrapped as BOS, document, end-of-text. Optimisation is identical across arms: sequence length 2{,}048, global batch 256 sequences (0.52M tokens), AdamW ($\beta = (0.9, 0.95)$, weight decay 0.1), gradient clipping at 1.0, 1\% linear warm-up, then cosine decay to 10\% of peak, bf16 autocast with fp32 master weights, on 4$\times$H200. Initialisation is deterministic and there is no dropout, so a second seed would change only the order of training windows and bf16 non-determinism. In the warm-up runs, stage 1 trains the whole tied embedding matrix, old rows included, at $10^{-3}$ on the first 10\% of the stream; stage 2 is a main run started from the stage-1 model.

\paragraph{Evaluation.} Every scored sequence is preceded by one context token (\S\ref{sec:setup}), identically for every arm of a model.
\begin{itemize}
  \setlength\itemsep{0pt}
  \item \textbf{BPB} on the held-out native slices (the first 4\,MB of each) and on FLORES-200 devtest. In the primary byte-matched version, each held-out document is split at whitespace into pieces of at most 2{,}000 bytes and every piece is scored independently, so every arm conditions on the same bytes. Token windows would give \RII{} more bytes of context per scored token. As a secondary metric, whole documents are scored with a sliding window of 2{,}048 tokens and stride 1{,}024.
  \item \textbf{Belebele}~\citep{bandarkar-etal-2024-belebele} \texttt{npi\_Deva} and \texttt{eng\_Latn}, zero-shot, scoring each answer string by its log-likelihood per byte given passage and question.
  \item \textbf{Translation}: FLORES-200 devtest Nepali$\leftrightarrow$English, 5-shot from FLORES dev, greedy decoding capped at 512 new tokens, first output line, chrF++.
  \item \textbf{Generation speed}: UTF-8 bytes generated per second in greedy continuation of FLORES dev Nepali sentences (batch 64, 128 new tokens).
\end{itemize}
Document ids were not stored with the pieces. The Nepali held-out slice holds 503 documents (the first 503 of FineWeb-2 \texttt{npi\_Deva} after cleaning; \url{results/revision/heldout_ne_doc_ids.json}), recovered after an external review, and all Nepali intervals resample these documents. An earlier version of this paper used a length heuristic (a piece shorter than 1{,}900 bytes ends a document, 484 clusters); switching moved interval endpoints by at most $0.0003$. English intervals still use that heuristic. The margin of 0.01 is 1.3\% of the base Llama model's Nepali BPB and 1.2\% of base Qwen's. The equal-compute checkpoints are scored on BPB only, on the first 1\,MB of each held-out slice (567 Nepali pieces from 114 documents).

\section{Data}
\label{app:data}

Table~\ref{tab:data} gives the corpus cleaning and the token count of each arm's training stream.

\begin{table}[t]
\centering\small
\setlength{\tabcolsep}{3.5pt}
\begin{tabular}{lrr}
\toprule
& Nepali & English \\
\midrule
documents read & 170,790 & 215,961 \\
dropped: under 200 chars & 0 & 0 \\
dropped: exact duplicate & 0 & 69 \\
dropped: FLORES 10-gram & 0 & 16 \\
training documents & 166,300 & 209,713 \\
held-out (MB) & 30 & 30 \\
\midrule
\multicolumn{3}{l}{Training tokens per arm (millions)}\\
Llama-3.2-1B \RO{} & 207 & 211 \\
Llama-3.2-1B \RI{} & 176 & 211 \\
Llama-3.2-1B \RII{} & 76 & 211 \\
Qwen3-0.6B \RO{} & 359 & 215 \\
Qwen3-0.6B \RI{} & 177 & 215 \\
Qwen3-0.6B \RII{} & 76 & 215 \\
\bottomrule
\end{tabular}
\caption{Corpus cleaning (FineWeb-2 \texttt{npi\_Deva}; FineWeb-Edu; documents placed in the 30\,MB held-out slice are not counted as training documents) and the token counts of each arm's identical 2\,GB training stream. English token counts differ across arms of a model by less than 0.012\%, from the few English documents that contain Devanagari or combining marks.}
\label{tab:data}
\end{table}

FineWeb-2 is already deduplicated and filtered for length, which is why our own filters remove almost nothing.

\section{Full audit}
\label{app:audit}

Table~\ref{tab:audit-full} lists every distinct tokenizer. For each of the nine aliases we compared token ids with its partner on all 1{,}012 devtest sentences of English, Nepali, Hindi, Chinese and Arabic; all were identical. Kimi-K3 (requires executing remote code) and Mistral-Large-3 (its transformers load does not round-trip; its \texttt{tokenizer.json} uses Mistral-NeMo's mark-aware regex) are not included. The four tokenizers that do not round-trip belong to encoders and translation models; we keep them as reference points.

\begin{table*}[t]
\centering\small
\setlength{\tabcolsep}{4pt}
\begin{tabular}{lllrrrrrl}
\toprule
Tokenizer & Aliases & Regex & $|V|$ & ne/en & 95\% CI & Floor & Mark-init. & Lossless \\
\midrule
IndicBERTv2 & -- & none & 250,000 & 1.09 & [1.08, 1.10] & -- & 6\% & no \\
Arkios & -- & mark-aware & 65,536 & 1.12 & [1.10, 1.13] & 0.79 & 15\% & yes \\
XLM-R & -- & none & 250,002 & 1.12 & [1.11, 1.14] & -- & 7\% & no \\
NLLB-200 & -- & none & 256,204 & 1.17 & [1.16, 1.18] & -- & 6\% & no \\
BLOOM & -- & other & 250,680 & 1.18 & [1.16, 1.19] & -- & 17\% & yes \\
mT5 & -- & none & 250,100 & 1.46 & [1.45, 1.47] & -- & 11\% & no \\
Gemma-3 & Gemma-4 & none & 262,145 & 1.53 & [1.51, 1.54] & -- & 9\% & yes \\
Sarvam-1 & -- & none & 68,096 & 1.54 & [1.52, 1.56] & -- & 14\% & yes \\
o200k & gpt-oss & mark-aware & 200,019 & 1.61 & [1.60, 1.63] & 0.80 & 26\% & yes \\
Llama-4 & -- & mark-aware & 201,135 & 1.83 & [1.81, 1.85] & 0.80 & 30\% & yes \\
Mistral-NeMo & Sarvam-M & mark-aware & 131,072 & 2.13 & [2.11, 2.15] & 0.79 & 33\% & yes \\
Gemma-2 & -- & none & 256,000 & 2.13 & [2.11, 2.15] & -- & 19\% & yes \\
Qwen3.5 & -- & mark-aware & 248,077 & 2.20 & [2.18, 2.23] & 0.80 & 19\% & yes \\
Llama-3 & SmolLM3 & letters & 128,256 & 2.58 & [2.55, 2.61] & 2.19 & 55\% & yes \\
TituLLM & -- & mark-aware & 170,497 & 2.60 & [2.57, 2.63] & 0.79 & 55\% & yes \\
DeepSeek-V3 & DeepSeek-V4 & mark-aware & 128,815 & 2.97 & [2.93, 3.00] & 0.80 & 38\% & yes \\
GLM-5.3 & -- & letters & 154,856 & 4.35 & [4.29, 4.40] & 2.19 & 33\% & yes \\
Qwen2.5 & Qwen3 & letters & 151,665 & 4.42 & [4.36, 4.47] & 2.17 & 32\% & yes \\
Mistral-7B & -- & none & 32,768 & 4.44 & [4.38, 4.50] & -- & 29\% & yes \\
GLM-4.5 & -- & letters & 151,365 & 4.76 & [4.71, 4.82] & 2.19 & 30\% & yes \\
cl100k & Phi-4, OLMo-2, Granite-4.1 & letters & 100,277 & 4.76 & [4.71, 4.82] & 2.19 & 30\% & yes \\
Falcon3 & -- & letters & 131,072 & 5.07 & [5.00, 5.14] & 3.96 & 28\% & yes \\
LFM2 & -- & letters & 64,400 & 5.77 & [5.70, 5.84] & 2.15 & 24\% & yes \\
GPT-2 & -- & letters & 50,257 & 7.55 & [7.46, 7.65] & 3.23 & 19\% & yes \\
\bottomrule
\end{tabular}
\caption{All 24 distinct tokenizers (33 released tokenizers) on FLORES-200 devtest. Aliases: tokenizers whose token ids are identical on every devtest sentence of five languages (Appendix~\ref{app:audit}). Regex: letters-only, mark-aware, other, or none (no regex pre-tokenizer). ne/en: Nepali/English token premium with 95\% paired-bootstrap interval. Floor: pre-token floor under the tokenizer's own regex, a lower bound on the premium of any merge-based extension (letters-only and mark-aware tokenizers only). Mark-init.: share of Nepali tokens whose first character is a combining mark. Lossless: decoding reproduces the input on 200 Nepali sentences. Arkios is from \citet{regmi2026arkios}. Hugging Face repositories are listed in the released code.}
\label{tab:audit-full}
\end{table*}

\paragraph{Broken syllables.} The share of Nepali tokens that begin with a combining mark is 55\% for Llama-3, 32\% for Qwen3 and 30\% for \texttt{cl100k}. Mark-aware tokenizers show similar rates (26\% for \texttt{o200k}, 55\% for TituLLM, which keeps Llama-3's merges), because BPE also splits within words.

\section{Premium decomposition}
\label{app:decomp}

We write the premium as
\begin{equation*}
\frac{|T(\text{ne})|}{|T(\text{en})|} = \underbrace{\frac{|P_{\text{rep}}(\text{ne})|}{|T(\text{en})|}}_{\text{base}} \cdot \underbrace{\frac{|P(\text{ne})|}{|P_{\text{rep}}(\text{ne})|}}_{\text{regex factor}} \cdot \underbrace{\frac{|T(\text{ne})|}{|P(\text{ne})|}}_{\text{vocabulary factor}},
\end{equation*}
where $P$ is the tokenizer's own pre-tokenizer and $P_{\text{rep}}$ is the repaired Llama-3 pattern, applied to every tokenizer so that the numerator of the base term is the same for all. The base lies between $0.77$ and $0.80$ because English token counts barely differ. Using one pattern means that Qwen's regex factor also includes its single-digit splitting. The regex factor is $0.99$--$1.02$ for mark-aware tokenizers, $2.77$--$2.79$ for the Llama-3, Qwen and \texttt{cl100k} pattern, $4.06$ for GPT-2, whose pattern puts every run of marks in a pre-token of its own, and $5.01$ for Falcon3. Falcon3's pipeline splits digits after its byte-level step, and its byte-level alphabet writes byte \texttt{0xBE} as \textthreequarters{}, which Unicode classes as a number, so it cuts the vowel sign \texttt{U+093E} inside its UTF-8 encoding. All floors and factors here count pre-tokens of each tokenizer's full pre-tokenizer pipeline; counting its last regex alone, as our audit script first did, understates Falcon3's floor ($3.21$ instead of $3.96$). The decomposition is an accounting identity and does not attribute cause. The regex factor is the room that a better regex creates, which new merges must then fill.

Table~\ref{tab:decomp} shows the factors. For Llama-3 the vocabulary factor is $1.18$, the lowest in the audit: its 128k vocabulary already nearly covers Nepali pre-tokens. For Qwen and \texttt{cl100k} both factors matter. Qwen3.5's word branch is our repaired class with only the leading-character class left unchanged; \citet{land2026smallbugs} reported its change for Hindi.

\begin{table}[!htb]
\centering
\small
\setlength{\tabcolsep}{4pt}
\begin{tabular}{lrrrr}
\toprule
Tokenizer & Premium & Base & Regex & Vocab. \\
\midrule
GPT-2 & 7.55 & 0.80 & 4.06 & 2.34 \\
LFM2 & 5.77 & 0.78 & 2.77 & 2.68 \\
Falcon3 & 5.07 & 0.79 & 5.01 & 1.28 \\
\texttt{cl100k} (+3) & 4.76 & 0.79 & 2.77 & 2.17 \\
GLM-4.5 & 4.76 & 0.79 & 2.77 & 2.17 \\
Qwen2.5 (+1) & 4.42 & 0.78 & 2.79 & 2.03 \\
GLM-5.3 & 4.35 & 0.79 & 2.77 & 1.98 \\
Llama-3 (+1) & 2.58 & 0.79 & 2.77 & 1.18 \\
\midrule
DeepSeek-V3 (+1) & 2.97 & 0.80 & 1.01 & 3.71 \\
TituLLM & 2.60 & 0.79 & 0.99 & 3.31 \\
Qwen3.5 & 2.20 & 0.78 & 1.02 & 2.76 \\
Mistral-NeMo (+1) & 2.13 & 0.77 & 1.02 & 2.69 \\
Llama-4 & 1.83 & 0.80 & 1.00 & 2.30 \\
\texttt{o200k} (+1) & 1.61 & 0.80 & 1.00 & 2.01 \\
Arkios & 1.12 & 0.79 & 1.01 & 1.41 \\
\bottomrule
\end{tabular}
\caption{Premium = base $\times$ regex factor $\times$ vocabulary factor, on FLORES-200 devtest, for all 15 letters-only and mark-aware tokenizers in the audit (top: letters-only; bottom: mark-aware; ``(+$n$)'' counts aliases). Merge-based extension can only lower the last column, and not below 1.}
\label{tab:decomp}
\end{table}

\section{Cross-script sweep}
\label{app:xling}

\begin{table}[!htb]
\centering
\small
\setlength{\tabcolsep}{3pt}
\begin{tabular}{lr|rrr|rrr}
\toprule
& & \multicolumn{3}{c|}{Llama-3} & \multicolumn{3}{c}{Qwen3} \\
Lang. & Marks & \RO{} & \RI{} & \RII{} & \RO{} & \RI{} & \RII{} \\
\midrule
Hindi & 31\% & 2.52 & 2.38 & 1.17 & 4.42 & 2.36 & 1.16 \\
Bengali & 32\% & 5.84 & 2.29 & 1.01 & 5.03 & 2.26 & 1.00 \\
Tamil & 35\% & 7.65 & 2.67 & 1.02 & 6.11 & 2.64 & 1.02 \\
Telugu & 36\% & 8.29 & 2.45 & 1.02 & 6.99 & 2.43 & 1.01 \\
Sinhala & 28\% & 8.63 & 2.30 & 1.20 & 6.88 & 2.33 & 1.24 \\
Burmese & 50\% & 11.62 & 3.02 & 1.21 & 8.93 & 2.99 & 1.11 \\
Thai$^\dagger$ & 21\% & 2.19 & 1.49 & 0.95 & 2.55 & 1.48 & 0.88 \\
\midrule
Amharic & 0\% & 7.59 & 0.98 & 0.98 & 4.05 & 0.98 & 0.98 \\
Korean & 0\% & 1.49 & 0.98 & 0.98 & 1.63 & 0.98 & 0.98 \\
\bottomrule
\end{tabular}
\caption{Language/English premium on FLORES-200 devtest, $K{=}32{,}000$. Marks: share of NFC code points that are combining marks; Amharic and Korean have none (controls). $^\dagger$Thai is written without spaces, so its pre-tokens are long and \RI{} is not yet at its floor. For \RI{}, merge learning stopped more than 100 merges short of 32{,}000 for Tamil, Telugu, Sinhala, Burmese and Hindi (Llama) (22{,}325--31{,}815 merges), because no further merge was available under the letters-only regex; Korean and Thai on Llama end at most 7 merges short.}
\label{tab:xling}
\end{table}

For each language we used the FineWeb-2 \emph{test} split, normalised to NFC, with documents that share a whitespace 10-gram with that language's FLORES-200 dev or devtest removed (for Thai we also checked 50-character overlaps and found none). The merge-learning text ranges from 25\,MB (Amharic) to 361\,MB (Bengali), less than the 1\,GB used for Nepali, so for languages where \RI{} ran out of merges we cannot fully separate the floor from data scarcity; Sinhala and Amharic \RII{} may also be limited by data.

\paragraph{A relaxed merge filter.} The Nepali filter keeps a new merge only if its output contains a complete Devanagari character. Scripts that the base vocabularies cover poorly have characters split into byte pieces, and this filter would drop the first byte-level merge of each such character and every merge built on it (for Sinhala \RII{} on Llama-3, only 13{,}530 of 32{,}000 merges survived). For this sweep we therefore also keep merges whose output is a byte piece of a target-script character. Such merges can fire on other scripts that share those bytes: the lead byte 0xE0 begins every code point from U+0800 to U+0FFF, which covers all Indic scripts and Thai, and a continuation-byte suffix can end code points in other blocks. They cannot fire on ASCII text, which contains no such bytes, and we verified English token ids on all 1{,}012 sentences for every tokenizer. Proposition~\ref{prop:lossless} does not cover these tokenizers, and we did not measure their effect on other non-ASCII scripts. Hindi shares the Nepali script, and Nepali merges also change Hindi tokenization.

\paragraph{Thai and Burmese.} Both are written without spaces between words, so a repaired pre-token can span several words. In practice no \RII{} Thai- or Burmese-script token spans a space in either language. In Thai, 28\% (Llama-3) and 25\% (Qwen3) of \RII{}'s Thai-script tokens cover more than one dictionary word (PyThaiNLP \texttt{newmm} segmentation), fewer than under \RI{} (33\% and 31\%). In Burmese, \RII{}'s tokens average 2.2 and 1.8 syllables, against 1.3 under \RI{} and 5.1 syllables per space-delimited chunk.

\section{Released extensions}
\label{app:released}

We did not follow a fixed search protocol. Candidates came from the papers in \S\ref{sec:related} and from manual searches of Hugging Face for Llama-3, Qwen and LFM2 derivatives in Indic and Southeast Asian languages. We kept releases whose tokenizer adds merges to a letters-only base. After inspection we excluded five releases that keep the base Llama-3 tokenizer unchanged and three that do not extend a letters-only base. The Llama-3 Hindi extension of Nanda~\citep{choudhury2025nanda} is gated and could not be measured, and we found no public tokenizer for \citet{nag2024efficientcpt} or \citet{purason-etal-2026-teaching}. The Yamaguchi et al.\ extensions re-order the base merges and are not English-invariant (English tokens rise by 6.5\% for Llama-3 Telugu and 7.8\% for Qwen3 Telugu), so Table~\ref{tab:released} compares tokens with pre-tokens of the target text directly. Two releases do not obey the letters-only floor: TituLLM~\citep{nahin-etal-2025-titullms} changes the regex, and SinLlama~\citep{aravinda2025sinllama} registers 11{,}080 strings as added tokens and reaches 0.54 times its letters-only floor.

\begin{table}[!htb]
\centering
\small
\setlength{\tabcolsep}{3pt}
\begin{tabular}{llrrr}
\toprule
Base & Lang. & Tokens/floor & Gap closed & Room \\
\midrule
Llama-3-8B & te & 1.016 & 99.3\% & 3.16$\times$ \\
Llama-3-8B & si & 1.021 & 99.3\% & 2.42$\times$ \\
Llama-3-8B & my & 1.013 & 99.5\% & 5.87$\times$ \\
\midrule
Llama-3.1-8B & ta & 1.018 & 99.1\% & 3.58$\times$ \\
Llama-3.1-8B & bn & 1.020 & 98.7\% & 2.74$\times$ \\
Llama-3.1-8B & gu & 1.026 & 98.9\% & 2.52$\times$ \\
Qwen2.5-7B & si & 1.019 & 99.1\% & 2.40$\times$ \\
Qwen2.5-7B & my & 1.036 & 98.2\% & 5.72$\times$ \\
Qwen3-14B & bn & 1.019 & 98.5\% & 2.71$\times$ \\
Qwen3-14B & te & 1.027 & 98.6\% & 3.11$\times$ \\
\midrule
LFM2 (2.5) & hi & 1.068 & 95.5\% & 2.24$\times$ \\
LFM2 (2.5) & ne & 1.160 & 90.5\% & 2.77$\times$ \\
LFM2 (2.5) & bn & 1.243 & 92.4\% & 2.74$\times$ \\
LFM2 (2.5) & th & 1.819 & 87.9\% & 5.57$\times$ \\
\midrule
Llama-3.1-8B & am & 1.573 & 93.8\% & 1.00$\times$ \\
\bottomrule
\end{tabular}
\caption{Released merge-based extensions that keep the letters-only regex, on FLORES-200 devtest of the target language. Tokens/floor: extended-tokenizer tokens divided by pre-tokens under its own regex (at least 1 by construction). Gap closed: share of the base-to-floor gap that the extension removed. Room: pre-tokens under the letters-only regex divided by pre-tokens under the repaired regex. First block: \citet{yamaguchi2026lowres}; second: \citet{yamaguchi2025elchat}, with the two Qwen3-14B rows from a later release by the same authors; third: LFM2.5~\citep{smith2026inplace}; last: the Amharic release of \citet{yamaguchi2025elchat}, a script with no combining marks (control).}
\label{tab:released}
\end{table}

\section{Additional tokenizer arms}
\label{app:extras}

\paragraph{Regex alone.} Repairing the regex without adding merges (\RII{}, $K{=}0$) gives $2.61$ for Llama-3 (from $2.58$) and $4.42$ for Qwen3. Under the repaired regex the old merges still tokenize Nepali words in old pieces; the repair only removes the bound that new merges would otherwise hit.

\paragraph{Added tokens.} Table~\ref{tab:added} gives the added-token arm of \S\ref{sec:floor}. Default matching reaches almost \RII{}'s compression and the lowest mark-initial rate, but splits many words at arbitrary points: a fifth of the added-token matches start inside a word and more than a quarter end inside one. With whole-word matching, few strings ever match and the premium returns close to the original.

\begin{table}[!htb]
\centering\small
\setlength{\tabcolsep}{3pt}
\begin{tabular}{lrrrr}
\toprule
 & \multicolumn{2}{c}{Premium} & \multicolumn{2}{c}{Mark-initial} \\
\cmidrule(lr){2-3}\cmidrule(lr){4-5}
Tokenizer & Llama-3 & Qwen3 & Llama-3 & Qwen3 \\
\midrule
\RO{} & 2.583 & 4.415 & 55\% & 32\% \\
\RI{}, $K{=}32$k & 2.204 & 2.187 & 65\% & 64\% \\
\RII{}, $K{=}32$k & 1.014 & 1.014 & 13\% & 11\% \\
\RII{} strings, added & 1.025 & 1.009 & 7\% & 7\% \\
\quad whole words only & 2.486 & 4.219 & 54\% & 32\% \\
\midrule
Added-token share & 77\% & 83\% & & \\
Starts mid-word & 21\% & 21\% & & \\
Ends mid-word & 28\% & 28\% & & \\
\bottomrule
\end{tabular}
\caption{Added tokens on the unchanged letters-only regex, FLORES-200 devtest. We register the 31{,}753 of \RII{}'s 32{,}000 new strings that decode to whole Devanagari text (with at most one leading space) as added tokens. Premium: Nepali tokens / English tokens. Mark-initial: share of Nepali tokens that start with a combining mark. The bottom rows describe default matching, which ignores word boundaries: the share of tokens that are added tokens, and the share of added-token matches that start or end inside a word. No sentence without a Devanagari character changes its token ids (191{,}460 FLORES-200 sentences, both models, both matching modes).}
\label{tab:added}
\end{table}

\paragraph{Broken syllables.} A token that starts with a combining mark has certainly split a syllable, but the converse does not hold: a token can also end inside an akshara, for example just before a virama, so these shares are a lower bound on split syllables. Counting only tokens that contain Devanagari, the share that begin with a combining mark rises from 57\% (Llama-3) and 33\% (Qwen3) to 67\% for both under \RI{}, and falls to 15\% and 12\% under \RII{}. The letters-only regex leaves few distinct Nepali pre-tokens, so 5k to 30k merges cover nearly all frequent ones, which is why \RI{} saturates early.

\begin{figure}[!htb]
\centering
\includegraphics[width=0.95\columnwidth]{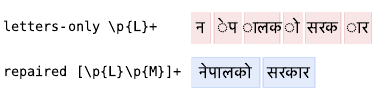}
\caption{Pre-tokens of \emph{nep\=alko sark\=ar} (``government of Nepal'') under the letters-only regex of Llama-3 and Qwen and under the repaired regex. We draw a dotted circle where a pre-token starts with a vowel sign; it is not a duplicated character.}
\label{fig:split}
\end{figure}

\section{All continued-pretraining runs}
\label{app:cptfull}

\begin{table*}[t]
\centering\footnotesize
\setlength{\tabcolsep}{3.5pt}
\begin{tabular}{lllrrrrrrrrr}
\toprule
 & & & & \multicolumn{2}{c}{held-out ne BPB$\downarrow$} & en BPB$\downarrow$ & FLORES & Belebele & \multicolumn{2}{c}{chrF++} & Gen. \\
\cmidrule(lr){5-6}\cmidrule(lr){10-11}
Model & Runs & Arm & Tokens & reg. & sep. & reg. & ne BPB & ne acc. & ne$\to$en & en$\to$ne & B/s \\
\midrule
Llama-3.2-1B & none & base & 0 & 0.766 & n/a & 0.788 & 0.863 & 0.277 & 32.4 & 10.8 & 5327 \\
 & main & \RO{} & 417M & 0.474 & 0.471 & 0.799 & 0.611 & 0.287 & 32.3 & 19.8 & 5147 \\
 &  & \RI{} & 386M & 0.534 & 0.483 & 0.803 & 0.648 & 0.268 & 27.2 & 10.4 & 6354 \\
 &  & \RII{} & 286M & 0.539 & 0.510 & 0.799 & 0.671 & 0.266 & 27.8 & 12.3 & 9398 \\
\addlinespace
 & warm-up & \RO{} & 417M & 0.779 & 0.471 & 0.850 & 0.864 & 0.279 & 8.8 & 9.1 & 5349 \\
 &  & \RI{} & 386M & 0.808 & 0.480 & 0.854 & 0.904 & 0.271 & 9.9 & 4.2 & 7521 \\
 &  & \RII{} & 286M & 0.527 & 0.501 & 0.840 & 0.652 & 0.283 & 16.2 & 15.4 & 11078 \\
\addlinespace
 & 1\,GB, BOS & \RO{} & 418M & 0.471 & 1.017 & 0.799 & 0.607 & -- & -- & -- & -- \\
 & in training & \RI{} & 386M & 0.483 & 0.875 & 0.800 & 0.617 & -- & -- & -- & -- \\
 &  & \RII{} & 286M & 0.509 & 1.666 & 0.798 & 0.626 & -- & -- & -- & -- \\
\addlinespace
\midrule
Qwen3-0.6B & none & base & 0 & 0.852 & 0.852 & 0.873 & 0.945 & 0.290 & 26.0 & 6.8 & 2186 \\
 & main & \RO{} & 574M & 0.501 & 0.501 & 0.864 & 0.631 & 0.270 & 20.1 & 13.6 & 3089 \\
 &  & \RI{} & 392M & 0.561 & 0.561 & 0.864 & 0.689 & 0.253 & 11.6 & 6.9 & 4972 \\
 &  & \RII{} & 292M & 0.570 & 0.570 & 0.862 & 0.688 & 0.266 & 17.0 & 3.9 & 7454 \\
\addlinespace
\bottomrule
\end{tabular}
\caption{Main, warm-up and BOS rerun (seed 0, learning rate $10^{-4}$), with secondary metrics; second seeds, matched-steps and learning-rate runs are in Table~\ref{tab:extra-runs}. Main runs: 1\,GB of Nepali and 1\,GB of English, documents separated by the end-of-text token only. Warm-up runs: the same, after training only the embedding matrix on the first 10\% of the stream (Tokens excludes that stage). Held-out BPB is scored in pieces of about 2{,}000 bytes, after the registered prefix (reg.: beginning-of-sequence token for Llama) or the training separator (sep.: end-of-text token); for Qwen the two coincide. The base Llama model was never trained without its beginning-of-sequence token, so its separator score is not meaningful (n/a). All columns other than sep.\ use the registered prefix; for the Llama warm-up runs they are dominated by the scoring artifact of \S\ref{sec:cpt} and should not be compared with the main runs. 1\,GB, BOS in training: the BOS-consistent rerun, i.e.\ the main runs repeated with every training document wrapped as BOS, document, end-of-text (third amendment); its sep.\ column is scored after the end-of-text token, which these models never saw directly before text. Belebele, translation and generation speed were not run for it (--). Belebele: zero-shot accuracy on 900 items (chance 0.25); every value lies between 0.25 and 0.29 with 95\% intervals of about $\pm$0.03, so it cannot separate the arms. chrF++: 5-shot FLORES-200 devtest translation, single greedy pass, no intervals, keeping only the first line of output (empty if the model starts with a newline), so we draw no conclusions from it. Gen.\ B/s: UTF-8 bytes of Nepali generated per second, measured on GPUs shared with concurrent training runs; it is noisy (the base and \RO{} Qwen models share tokenizer and architecture yet differ by 41\%) and we do not draw conclusions from it.}
\label{tab:cpt}
\end{table*}

\paragraph{Equal compute and training time.} At \RII{}'s final step count, the \RO{} checkpoint of the main runs, whose learning rate had not yet decayed, is already better than the final \RII{} model by $0.060$ $[0.056, 0.063]$ on Llama and $0.053$ $[0.049, 0.057]$ on Qwen. The \RI{} checkpoint at that step is better by $0.006$ $[0.003, 0.008]$ on Llama and level on Qwen ($0.0000$ $[-0.0021, +0.0020]$; final \RI{} is better than final \RII{} by $0.009$ on the same pieces, so on Qwen all of \RI{}'s advantage arrives in its last 192 steps, over which its learning rate falls the last 16\% of the way to its floor). These are scored on the first 1\,MB of held-out text with the registered prefix, so the Llama values carry the BOS artifact of \S\ref{sec:cpt}; no such checkpoints were scored for the rerun. Measured training time fell by only 3\% (Llama) and 11\% (Qwen) relative to \RI{}, in runs that shared GPUs with evaluation jobs.

\paragraph{Warm-up runs.} Training only the embedding matrix first leaves \RO{} where it was ($0.471$ after the training separator) and improves \RII{} from $0.510$ to $0.501$. \RII{}'s deficit relative to \RO{} shrinks from $0.039$ to $0.031$, as the second amendment predicted, and its gap to \RI{} from $0.027$ to $0.021$ $[0.020, 0.022]$. These are differences between single runs, so we treat them as suggestive. The amendment stated its prediction for the registered prefix, where the warm-up makes Llama scores unusable: \RO{} moves from $0.474$ to $0.779$. With tied embeddings, the BOS row is updated as an output class even though it never appears as input, and in the warm-up runs its norm grows from $0.863$ to $0.896$; this may contribute.

\paragraph{English.} In the main Llama runs under the registered prefix, \RII{} $-$ \RO{} on English is $+0.0003$ $[-0.0002, +0.0008]$ (non-inferior). All three main-run arms are slightly worse on English than the base model ($0.799$ to $0.803$ against $0.788$).

\paragraph{Other BPB views.} The registered sliding-window BPB, which scores whole documents, orders the Nepali arms the same way (Llama rerun $0.424$, $0.437$, $0.464$; Qwen $0.459$, $0.515$, $0.525$). In the BOS-consistent rerun, FLORES Nepali BPB is $0.607$, $0.617$ and $0.626$ for \RO{}, \RI{} and \RII{}; the rerun's three held-out differences match the separator diagnostic of the main runs ($+0.0121$, $+0.0394$, $+0.0273$), and \RII{} $-$ \RI{} lies between the two earlier values, as the third amendment predicted. At $10^{-4}$, every Qwen arm ends better on English than its base model ($0.862$ to $0.864$ against $0.873$).

\paragraph{A failed prediction of the third amendment.} The third amendment also predicted that scores after BOS and after the end-of-text token would agree within $0.005$. This failed: after the end-of-text token the rerun models score $1.017$, $0.875$ and $1.666$. Score a model after the context it was trained with.

\paragraph{English gains.} \RII{} is better than \RI{} on English by $0.0022$ on both models although English tokens are identical; a likely reason is that English is 74\% of \RII{}'s training tokens but 55\% of \RI{}'s, so it weighs more in \RII{}'s loss.

\paragraph{BOS diagnostic.} On 40 held-out Nepali documents (first 1{,}500 characters), Llama BPB after BOS and after the end-of-text token are 0.720 and 1.865 for the base model, 0.441 and 0.438 for \RO{} of the main runs, and 0.758 and 0.438 for \RO{} of the warm-up runs. The norm of the BOS embedding is 0.863, 0.860 and 0.896 respectively (mean row norm 0.98--0.99). The base model, pretrained with BOS at sequence starts, scores far worse without it, while the continued-pretraining models, trained on fixed-length windows that start at arbitrary tokens, do not need it. The BOS-consistent rerun shows the converse: trained with BOS before every document, its models score far worse after the end-of-text token (Table~\ref{tab:cpt}, sep.\ column). We did not test the mechanism further.

\paragraph{Secondary metrics.} Belebele is near chance for every model. Translation scores are erratic: Qwen \RII{} scores 3.9 chrF++ from English to Nepali, below the untrained base model (6.8). We take the first output line as the hypothesis, which is empty when a model starts with a newline, and we did not measure how often that happens. Qwen's FLORES BPB is almost the same for \RI{} and \RII{} (0.689 and 0.688), although \RII{} is worse on held-out native text. We draw no conclusions from these metrics.

\paragraph{Measured generation speed.} Generation was timed once per model, on GPUs that were also running training or evaluation jobs. The Qwen base model and Qwen \RO{} share tokenizer and architecture, yet differ by 41\% in bytes per second, and the generated text has fewer bytes per token than FLORES (8.1 against 12.3 for Llama \RII{}), so the measured speed also depends on what each model chose to write. With these caveats, \RII{} generates $1.48$ (Llama) and $1.50$ (Qwen) times as many bytes per second as \RI{}, and its generations carry $8.07$ and $7.47$ bytes per token against $5.21$ and $5.07$ for \RI{}. Section~\ref{sec:cpt} reports these together with the deterministic ratio of decoding steps on reference text.

\section{What the repair does not fix}
\label{app:notfix}
A wider word class removes the splits that the regex forces, but BPE can still split inside a syllable: 26\% of \texttt{o200k}'s Nepali tokens and 11 to 13\% of \RII{}'s start with a combining mark. An akshara-aware rule~\citep{kapila2026typedriven,darshana2026wwho} restricted to the Devanagari block would keep English token ids by the argument of Proposition~\ref{prop:lossless}; testing one in a deployed model is future work. Added tokens are a second route past the floor that needs no regex change; they split more words mid-way (\S\ref{sec:floor}), but we have not measured their effect on model quality, and our training runs give no evidence that segmentation of this kind matters for it.

\section{All runs of the later amendments}
\label{app:extra-runs}
Table~\ref{tab:extra-runs} lists every run of the Llama BOS setting and of Qwen. Matched-steps Llama \RII{} is slightly worse on English than the other Llama arms ($0.804$ against $0.798$ to $0.800$), possibly because it sees part of its stream twice.

\begin{table}[!htb]
\centering\small
\setlength{\tabcolsep}{2pt}
\begin{tabular}{lllrrrr}
\toprule
Model & Arm & LR & Seed & Steps & ne BPB & en BPB \\
\midrule
Llama & \RO{} & $10^{-4}$ & 0 & 797 & 0.4709 & 0.7985 \\
Llama & \RO{} & $10^{-4}$ & 1 & 797 & 0.4709 & 0.7984 \\
Llama & \RI{} & $10^{-4}$ & 0 & 737 & 0.4833 & 0.7998 \\
Llama & \RI{} & $10^{-4}$ & 1 & 737 & 0.4832 & 0.7993 \\
Llama & \RII{} & $10^{-4}$ & 0 & 546 & 0.5089 & 0.7975 \\
Llama & \RII{} & $10^{-4}$ & 1 & 546 & 0.5080 & 0.7974 \\
Llama & \RII{}-m & $10^{-4}$ & 0 & 737 & 0.4957 & 0.8044 \\
Llama & \RII{}-m & $10^{-4}$ & 1 & 737 & 0.4948 & 0.8043 \\
\midrule
Qwen & \RO{} & $10^{-4}$ & 0 & 1095 & 0.5007 & 0.8643 \\
Qwen & \RO{} & $10^{-4}$ & 1 & 1095 & 0.5003 & 0.8631 \\
Qwen & \RI{} & $10^{-4}$ & 0 & 748 & 0.5609 & 0.8637 \\
Qwen & \RI{} & $10^{-4}$ & 1 & 748 & 0.5578 & 0.8638 \\
Qwen & \RII{} & $10^{-4}$ & 0 & 556 & 0.5702 & 0.8615 \\
Qwen & \RII{} & $10^{-4}$ & 1 & 556 & 0.5729 & 0.8617 \\
Qwen & \RII{}-m & $10^{-4}$ & 0 & 748 & 0.5491 & 0.8628 \\
Qwen & \RII{}-m & $10^{-4}$ & 1 & 748 & 0.5499 & 0.8626 \\
Qwen & \RO{} & $3{\times}10^{-4}$ & 0 & 1095 & 0.4794 & 0.8789 \\
Qwen & \RI{} & $3{\times}10^{-4}$ & 0 & 748 & 0.5092 & 0.8748 \\
Qwen & \RII{}-m & $3{\times}10^{-4}$ & 0 & 748 & 0.5035 & 0.8751 \\
\bottomrule
\end{tabular}
\caption{Every continued-pretraining run of the Llama BOS setting and of Qwen, including the matched-steps, second-seed and learning-rate runs of the fourth to seventh amendments. BPB: held-out text in pieces of about 2{,}000 bytes after the registered prefix. \RII{}-m: \RII{} trained for \RI{}'s step count, with 35\% of its stream seen twice. Seeds change data order (and bf16 non-determinism) only.}
\label{tab:extra-runs}
\end{table}

\end{document}